\documentclass{article}

\usepackage{times}
\usepackage{graphicx} 

\usepackage{natbib}

\usepackage{algorithm}
\usepackage{algorithmic}

\usepackage{hyperref}

\usepackage{subfig}
\usepackage{amsmath,amsfonts,amssymb,amsthm}
\usepackage{multirow}
\usepackage{bbold}
\usepackage{listings}
\usepackage{fixltx2e}
\usepackage{color}
\usepackage{xcolor}
\usepackage{xspace}
\usepackage{booktabs}

\DeclareMathOperator*{\argmin}{arg\,min}
\definecolor{mygreen}{rgb}{0,0.6,0}

\newcommand{\ouralgo}{\textsc{SymSGD}\xspace}

\newcommand{\ouralgosec}{SymSGD\xspace}

\newcommand{\hogwild}{\textsc{Hogwild!}\xspace}
\newcommand{\allreduce}{\textsc{AllReduce}\xspace}
\newcommand{\vowpal}{Vowpal Wabbit\xspace}
\newcommand{\realR}{\mathbb{R}}
\newcommand{\covm}[1]{\mathbb{C}(#1)}

\newcommand{\norm}[1]{\left\lVert#1\right\rVert_2}

\newtheorem{theorem}{Theorem}[section]
\newtheorem{lemma}[theorem]{Lemma}
\newcommand{\para}[1]{{\bf #1}}
\newcommand{\comment}[1]{}

\usepackage[accepted]{icml2017}

\icmltitlerunning{Parallel Stochastic Gradient Descent with Sound Combiners}

\begin{document} 

\twocolumn[
\icmltitle{Parallel Stochastic Gradient Descent with Sound Combiners}



\icmlsetsymbol{equal}{*}

\begin{icmlauthorlist}
\icmlauthor{Saeed Maleki}{microsoft}
\icmlauthor{Madanlal Musuvathi}{microsoft}
\icmlauthor{Todd Mytkowicz}{microsoft}
\end{icmlauthorlist}

\icmlaffiliation{microsoft}{Microsoft Research}
\icmlcorrespondingauthor{\{saemal,madanm,toddm\}@microsoft.com}{}

\icmlkeywords{boring formatting information, machine learning, ICML}

\vskip 0.3in
]



\printAffiliationsAndNotice{}  

\newcommand{\Expect}{{\rm I\kern-.3em E}}
\newcommand{\Var}{{\rm Var}}

\newcommand{\MR}{MR-\ouralgo\xspace}
\newcommand{\AS}{Async-\ouralgo\xspace}
\newcommand{\HW}{HogWild!\xspace}

\begin{abstract} 
  Stochastic gradient descent (SGD) is a well-known method for
  regression and classification tasks. However, it is an inherently
  sequential algorithm---at each step, the processing of the current
  example depends on the parameters learned from the previous
  examples. Prior approaches to parallelizing linear learners using
  SGD, such as \hogwild and \allreduce, do not honor these
  dependencies across threads and thus can potentially suffer poor
  convergence rates and/or poor scalability. This paper proposes
  \ouralgo, a parallel SGD algorithm that, to a first-order
  approximation, retains the sequential semantics of SGD.  Each thread
  learns a local model in addition to a \emph{model combiner}, which
  allows local models to be combined to produce the same result as
  what a sequential SGD would have produced.
This paper evaluates \ouralgo's accuracy and performance on $6$
datasets on a shared-memory machine shows up-to $11\times$ speedup over
our heavily optimized sequential baseline on $16$ cores and $2.2\times$, on average, faster than \hogwild.
\end{abstract} 


\section{Introduction}
\label{sec:intro}
Stochastic Gradient Descent~(SGD) is an effective method for many
machine learning problems.  It is a simple algorithm with few
hyper-parameters and its convergence rates are well understood both
theoretically and empirically. However, its performance scalability is
severely limited by its inherently sequential computation. SGD
iteratively processes its input dataset where the computation at each
iteration depends on the model parameters learned from the previous
iteration.

Current approaches for parallelizing SGD learn local models per thread
and combine these models in ways that do not honor this inter-step
dependence. For instance, threads in \hogwild~\cite{hogwild} racily
update a shared global model without holding any locks. In
parameter-server~\cite{parameterServerOSDI}, each thread (or machine)
periodically sends its model deltas to a server that applies them to a
global model, even though the deltas were computed on a stale model from
a few updates ago. 

While these algorithms are guaranteed to eventually
converge, they need to carefully manage the communication-staleness
trade-off. On the one hand, \hogwild communicates after processing
every input example to achieve bounded staleness but the resulting
communication cost limits scalability even in a single machine on
\emph{sparse} datasets --- as we show in our experiments, even sparse
datasets have frequent features that produce significant cache
traffic. On the other hand, techniques such as
\allreduce~\cite{allreduce} the staleness causes a drop in accuracy
on the same number of examples with respect to a sequential baseline. 


\comment{
The key challenge in these algorithms is to manage the trade-off
between the cost of communication and keeping the local models up to
date with the global model. Our experiments show that techniques like
\hogwild do not scale, particularly across multiple processor sockets
due to the frequency of communication, while other techniques fall
short on accuracy with respect to a sequential baseline on the same
number of examples. 
}

This paper presents \ouralgo, a parallel SGD algorithm that allows the
threads to communicate less frequently but achieve a high-fidelity
approximation to what the threads would have produced had they run
sequentially. The key idea is for each thread to generate a sound
\emph{model combiner} that precisely captures the first-order effects
of a SGD computation starting from an \emph{arbitrary}
model. Periodically, the threads update a global model with their
local model while using the model combiner to account for changes in
the global model that occurred in the interim. 

While the algorithm can be generalized to different machine learning
problems and on different parallel settings, such as distributed
clusters and GPUs, we focus our evaluation on linear learners on
mulitcore machines. This is primarily motivated by the fact that this
setting forms the core of machine learning today.  At Microsoft, 
developers trained over 1
million models per month in 2016 on single-node
installations. Likewise, Databrick's 2015 survey showed almost 50\% of
Spark installations are single-node~\cite{databricks}. As machines
with terabytes of memory become a commonplace (As of February
2017, one can rent an X1 instance from AWS with
2 TB memory and 128 cores for less than \$4 per
hour~\cite{aws-x1-instance}), machine learning tasks on large datasets
can be done efficiently on single machines without paying the inherent
cost of distribution~\cite{frankmcsherry}. 

Our evaluation shows that \ouralgo is fast, scales well on multiple
cores, and achieves the same accuracy as sequential SGD. When compared
to our optimized sequential baseline, \ouralgo achieves a speedup of
$8.3X\times$ to $11\times$ on 16 cores. This represents a $2.25\times$
speedup over \hogwild, on average.



\section{Parallel \ouralgosec Algorithm}
\newcommand{\dotp}[2]{#1\cdot #2}
\newcommand{\deriv}[2]{\frac{\partial #1}{\partial #2}}

Given a set of $N$ input examples $z_i = (x_i,y_i)$, where $x_i$ is a
vector of $f$ feature values and $y_i$ is the label to learn, let
$C(w) = \frac{1}{N} \sum_i C_{z_i}(w,x_i,y_i)$ be the convex cost
function to minimize. That is, we seek to find $$w^*=\argmin_{w\in
\realR^f}\sum_{i=0}^n C_{z_i}(w, x_i, y_i)$$ The cost function can
optionally include a regularization term. We define $G \triangleq
\deriv{C}{w}$ and $G_z \triangleq
\deriv{C_z}{w}$ for the gradients, and $H \triangleq \deriv{G}{w}$ and $H_z \triangleq
\deriv{G_z}{w}$ for the Hessian of the cost function. 

\comment{
Consider a training dataset $(X_{n\times f},y_{n\times 1})$, where $f$
is the number of features, $n$ is the number of examples in the
dataset, the $i^{th}$ row of matrix $X$, $x_i$, represents the
features of the $i^{th}$ example, and $y_i$ is the dependent value (or
label) of that example. 

A linear model seeks to find
a $$w^*=\argmin_{w\in \realR^f}\sum_{i=0}^n Q(x_i\cdot w, y_i)$$ that
minimizes an error function $Q$. 

For linear learners, the gradient and the Hessian of the cost function
have a special form. Let $G_z(v,y) \triangleq \deriv{C_z(v,y)}{v}$ and
$H_z(v,y) \triangleq \deriv{G_z(v,y)}{v}$ be the partial derivatives
respectively of $C_z$ and $G_z$ with respect to $v =
\dotp{x}{w}$. Thus we have $G\triangleq\deriv{C}{w} = \frac{1}{N}\sum_{i} G_{z_i}(\dotp{x_i}{w},
y_i)\cdot x_i$ and $H\triangleq\deriv{G}{w} = \frac{1}{N}\sum_{i} H_{z_i}(\dotp{x_i}{w},
y_i)\cdot x_i \cdot x_i^T$. 
}

At each step $t$, SGD picks 
$z_r = (x_r, y_r)$ uniformly randomly from the input dataset 
and updates the current model
$w_t$ along the gradient $G_{z_r}$:
\begin{equation}
\label{eq:seqsgd}
w_{t+1}=w_t - \alpha_t G_{z_r}({w_{t}, x_r, y_r})
\end{equation}
Here, $\alpha_t$ is the {\em learning rate} that determines the
magnitude of the update along the gradient. As this equation shows,
$w_{t+1}$ is dependent on $w_t$ and this dependence 
makes
parallelization of SGD across iterations difficult.

\label{sec:algo}
\begin{figure}
  \centering
  \includegraphics[width=\columnwidth]{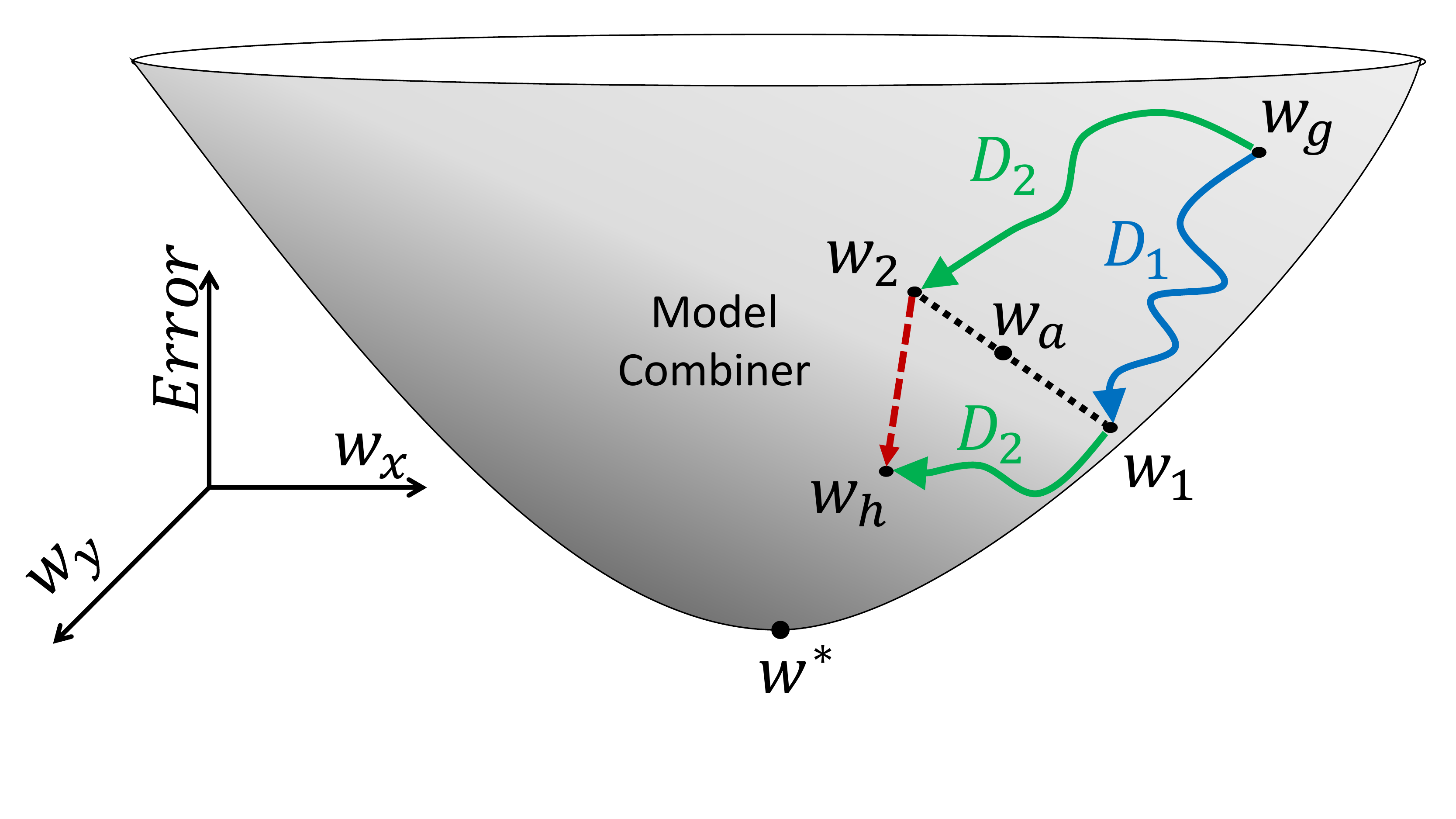}
  \caption{Convex error function for a two-dimensional feature space.}
  \label{fig:sgdcurve}
\end{figure}

Figure~\ref{fig:sgdcurve} demonstrates this difficulty. Say, a
processor performs SGD on a sequence of examples $D_1$ from a global
model $w_g$ to reach $w_1$. When processing a subsequent sequence
$D_2$, a sequential SGD algorithm would have started from $w_1$ to
reach $w_h$. Now, we desire to process $D_1$ and $D_2$ in parallel,
but the computation on $D_2$ cannot start on $w_1$, which is known 
only after the computation on $D_1$ has finished.

State of the art parallelization techniques such as \hogwild and
\allreduce approach this problem by processing $D_1$ and $D_2$
starting from the same model $w_g$ and respectively reaching their
local models $w_1$ and $w_2$. Then, they combine their local models
into a global model, but do so in an ad-hoc manner. For instance,
\allreduce computes a weighted average of $w_1$ and $w_2$, where the
per-feature weights are chosen so as to prefer the processor that has
larger update for that feature. This weighted average is depicted
pictorially as $w_a$ in Figure~\ref{fig:sgdcurve}. But doing so does
not necessarily reach $w_h$, the model that a sequential SGD would
have produced. \hogwild attempts to get around this staleness
problem by communicating frequently after every input example~(that
is, the size of $D_1$ and $D_2$ is $1$). But the resulting
communication cost hurts scalability particularly across multiple
sockets. This is true even for sparse datasets due to the presence of
frequently-occurring features.



\subsection{Symbolic SGD}
\label{sec:M}


The goal of this paper is to \emph{soundly} combine local models with
the hope of producing the same model as what a sequential SGD would
have produced. In Figure~\ref{fig:sgdcurve}, we seek a method to
combine $w_1$ and $w_2$ into the global model $w_h$. This requires
``adjusting'' the computation of $D_2$ for the staleness $w_1-w_g$ in
the starting model.

To do so, the second processor performs its computation from 
$w_g + \Delta w$, where $\Delta w$ is an \emph{unknown} symbolic
vector. This allows the second processor to both compute a local model
(resulting from the concrete part) and a model combiner (resulting
from the symbolic part) that accounts for changes in the initial
state. Once both processors are done learning, second processor finds
$w_h$ by setting $\Delta w$ to $w_1-w_g$ where $w_1$ is computed by
the first processor. This parallelization approach of SGD can be
extended to multiple processors where all processor produce a local
model and a combiner (except for the first processor) and the local
models are combined sequentially using the combiners.

\subsection{Model Combiners}
\label{sec:model:combiners}
Let $S_D(w)$ represent the SGD computation of dataset $D$ starting
from $w$. For example, $w_1 = S_{D_1}(w_g)$ in
Figure~\ref{fig:sgdcurve}. To generate the model combiner, we need to
reason about $S_D(w + \Delta w)$. 
Assuming that $S_D$ is differentiable, we have the following Taylor
series expansion: 
\begin{align}
S_{D} (w + \Delta w) = \underbrace{S_{D}(w)}_\text{local
  model} + \underbrace{S'_{D}(w)}_\text{model combiner}\cdot
\Delta w + O(|\Delta w|_2)
\label{tayloreqn}
\end{align}
We define  $M_D  \triangleq S'_D = \deriv{S}{w}$ as the model combiner. 
In the equation above, the model combiner captures the first-order effect of
how a $\Delta w$ change in $w_g$ will affect the SGD computation. For
instance, by using $\Delta w = w_1 - w_g$ in this equation, one can
combine the local models in Figure~\ref{fig:sgdcurve} to generate
$w_h$.

When $\Delta w$ is sufficiently small, one can neglect the second order
term and use the model combiner to combine local models with
sufficient fidelity. Section~\ref{sec:convproof} in the Appendix shows
that convergence is guaranteed when neglecting the higher order
terms under certain general assumptions of the cost function, provided
$\norm{\Delta w}$ is bounded.  

The following lemma shows how
to generate a model combiner. 

\begin{lemma}
\label{lemma:NM}
Let $D = (z_1, z_2, \ldots, z_n)$ be a sequence of input examples and 
$D_i$ represent the subsequence $(z_1, \ldots, z_i)$. The
model combiner is given by
\begin{equation}
\label{eqn:mc}
M_D(w) = \prod_{i=n}^1 (I-\alpha_i \cdot H_{z_i}(S_{D_{i-1}}(w), x_i, y_i)
\end{equation}
with $S_{D_0}(w) = w$
\end{lemma}
\begin{proof}
We have 
$$S_D(w) = S_{z_n}(S_{z_{n-1}}(\ldots(S_{z_1}(w))))$$
The proof follows from Equation~\ref{eq:seqsgd} and the chain rule. 
\end{proof}

\subsection{The Parallel SGD Algorithm}
\label{sec:parallel:algo}
Model combiners provide a lot of flexibility to design parallel SGD
algorithms. Section~\ref{sec:eval} explores both a map-reduce version
and an asynchronous version. We describe the former here for
completeness. 

In the {map} phase, each processor $i \in [1,N]$ starts
from the same global model $w_g$ and computes its local model
$S_{D_i}(w_g)$ and the model combiner $M_{D_i}(w_g)$ in parallel. A
subsequent \emph{reduction} phase combines the local models by
adjusting the input of processor $i$ by $w_{i-1} - w_g$. 

\begin{equation} 
\label{eq:combine}
w_i = S_{D_i}(w_g) + M_{D_i}(w_g) \cdot (w_{i-1} - w_g) 
\end{equation}

\comment{
Lemma~\ref{lemma:M} ensures that this combination of local models will
produce the same output as what these processors would have generated
had they run sequentially. 

Of course, this is just one possible implementation of the parallel
algorithm. Another option would be to perform the model combination
asynchronously where each processor independently decides to update
the global model with its current local model and model
combiner. 
}
\subsection{Examples}
Many interesting machine learning algorithms, such as linear
regression, linear regression with L2 regularization, and polynomial
regression have a linear update to the model parameters (but not
necessarily linear on the input example). In such cases, the higher
order terms in Equation~\ref{tayloreqn} vanish. For such learners,
model combiners generate \emph{exactly} the same model as a sequential
SGD. 

\comment{If the dependence on
$\Delta w$ is linear during an SGD update, which is indeed the case
for this class of algorithms, then the symbolic dependence on
$\Delta w$ on the final output is \emph{exactly}
$$S_{D_k} (w_s + \Delta w) = S_{D_k}(w_s) + S'_{D_k}(w_s)\cdot \Delta
w$$ because the higher order terms ($O(|\Delta w|_2)$) are 0.
}

Specifically, considering standard linear regression with square loss,
the combiner matrix is given by 
$$M_{D}(w) =  \prod_{i=n}^1 (I-\alpha_i \cdot x_i \cdot x_i^T)$$
when computing on $D = (x_1, y_1) \ldots (x_n, y_n)$. Since the model
combiner is independent of $w$, this can be computed once and reused
in subsequent phases provided the learning rates do not change. 

For logistic regression, which has the update rule
$$w_i = w_{i-1} - \alpha \cdot (\sigma(x_i \cdot w_{i-1}) - y_i)\cdot x_i$$ 
where $\sigma$ is the sigmoid function, the 
model combiner is given by
\begin{equation}
  \label{eq:general}
  M_{D}(w) = \prod_{i=n}^1(I- \alpha_i \cdot \sigma'(x_i \cdot w_{i-1})\cdot x_i\cdot x_i^T)
\end{equation}
where $w_{0} = w$.
The model combiner for logistic regression is the model combiner
generated for linear regression but with $\alpha$ scaled by
$\sigma'(x_i \cdot w_{i-1})$.

\begin{table*}[]
\centering
\caption{Model combiners for various linear learners}
\label{tab:model:combiners}
\begin{tabular}{@{}lll@{}}
\toprule
Algorithm & SGD Update for $z=(x,y)$ & Model Combiner $M_z(w,x,y,\alpha)$       \\ \midrule
OLS       & $w - \alpha (x\cdot w - y)$         & $I- \alpha \cdot x\cdot x^T$                           \\
Logistic  & $w - \alpha (\sigma(x\cdot w) - y)$ & $I- \alpha \sigma'(x \cdot w)\cdot x \cdot x^T$ \\
Perceptron& $w + \alpha (y\cdot x\cdot \delta_{y\cdot x\cdot w \leq 0} )$& $I$\\
SVM       & $w - \alpha (\lambda w - y\cdot x\cdot \delta_{x\cdot w >
1}) $& $(1 - \alpha \lambda) I$\\ 
Lasso     
& $ [max(0, u_i \triangleq w_i - \alpha(\lambda + s(i)(y - w \cdot x))x_i]_i $ 
& $ [\delta_{u_i > 0}(\delta_{i=j} - \alpha s(i) x_i x_j)]_{ij}$\\ 
\bottomrule
\end{tabular}
\caption{Model combiners for linear learners from
~\cite{bottou}. Here, $\lambda > 0$ is an additional
hyperparameter and $\delta_\phi$ is $1$ when $\phi$ is true else $0$. 
In Lasso, the model $w$ consists of positive $w_+$ and negative $w_-$
features with $s(i)$ denoting the sign of feature $i$. $[v_i]_i$
describes a vector with $v_i$ as the $i$th element and
$[m_{ij}]_{ij}$ represents a matrix with $m_{ij}$ as the
$(i,j)$th element.}
\end{table*}

Table~\ref{tab:model:combiners} provides the model combiners for a few
linear learners. When the SGD update function is not differentiable,
using the Taylor expansion in Equation~\ref{tayloreqn} can result in
errors at points of discontinuity. However, assuming bounded
gradients, these errors do not affect the convergence of
\ouralgo~(Section \ref{sec:convproof}).

\comment{Any implementation of this algorithm must compute both a local model
in addition to the model combiner matrix. In practice, an algorithm
drops the higher order terms.  The magnitude of such an approximation
is bounded by the norm of $\Delta w$.  Thus, frequent communication
implies a better approximation as $\Delta w$ grows with each new
example fed to the SGD algorithm.  As such, this algorithm lets a
practitioner balance parallel computation with communication.
}

\comment{
Real machine learning problems involve learning over tens of thousands
to billions of features. Computing and maintaining a model combiner
matrix for problems that large clearly is not feasible.  We solve this
problem through dimensionality reduction.  Johnson-Lindenstrauss~(JL)
lemma~\cite{jllemma} allows us to project a set of vectors from a
high-dimensional space to a random low-dimensional space while
preserving distances. We use this property to reduce the size of the
combiner matrix without losing the fidelity of the computation.
Section~\ref{sec:MA} discusses these details and Appendix:~\ref{}
demonstrates that under reasonably mild assumptions, our algorithm
converges.
}

\comment{
In general, any non-linear activation function, $f(x \cdot w, y)$, results in a model
combiner with the form
$$S'_{D_k}(w_s) = \prod_{i=b}^a (I-\alpha \cdot f'(x_i\cdot w_{i-1}, y_i)\cdot x_i^T \cdot x_i)$$
where $f'$ is the derivative of $f$ with respect to the model
parameters $w$.
}

\section{Dimensionality Reduction of a Model Combiner}
\label{sec:MA}

One key challenge in using model combiners as described above is that
they are large $f \times f$ matrices.  Machine learning problems
typically involve learning over tens of thousands to billions of
features. Thus, it is impossible to represent the model combiner
explicitly. This section describes mechanisms to address this
problem. 

The basic idea is to project the combiner matrix into a smaller
dimension while maintaining its fidelity. This projection is
inspired by the Johnson-Lindenstrauss~(JL) lemma~\cite{jllemma} and
follows the treatment of Achlioptas~\cite{verysparse}. While this
projection generates an unbiased estimate of the combiner, its
variance could potentially affect convergence. Our convergence proof
in Section~\ref{sec:convproof} show that with appropriate bounds on
this variance, convergence is guaranteed. 

\subsection{Random Projection}
We observe that the only use of a combiner matrix $M_D$ in \ouralgo is
to multiply it with a $\Delta w$. To avoid representing $M_D$
explicitly, we instead maintain $M_D \cdot A$ for a randomly generated
$f \times k$ matrix $A$ with $k \ll f$. Then we estimate $M_D \cdot
\Delta w$ with $M_D \cdot A \cdot A^T \cdot \Delta_w$. The following lemma describes
when this estimation is unbiased.

Let $[m_{ij}]_{ij}$ represents a matrix with $m_{ij}$ as the element
in the $i$th row and $j$th column. 
\begin{lemma}
  \label{lem:random}
Let $A = [a_{ij}]_{ij}$ be a random $f\times k$ matrix with 
$$a_{ij} = d_{ij}/\sqrt{k}$$
where $d_{ij}$ is independently sampled 
from a random distribution $D$ with $\Expect[D]=0$ and $\Var[D]=1$. Then 
$$\Expect[A\cdot A^T]=I_{f\times f}$$
\end{lemma}
\begin{proof}
If $B = [b_{ij}]_{ij} \triangleq A \cdot A^T$, we have 
$\Expect[b_{ij}] = \frac{1}{k} \sum_k \Expect[a_{ik}a_{jk}]$.
When $i\neq j$, $\Expect[b_{ij}] = 0$ as $a_{ik}$ and $a_{jk}$ are
independent random variables with mean $0$. $\Expect[b_{ii}] = 1$ as
the variance of $a_{ii}$ is $1$. 
\end{proof}


With this lemma, the model combination with Equation~\ref{eq:combine} becomes
\begin{equation}
  \label{eq:MA}
  w_i \approx S_{D_i}(w_{g}) + M_{D_i}(w_{g})\cdot A\cdot A^T(w_{i-1}-w_g) 
\end{equation}
This allows an efficient algorithm that only computes the projected
version of the combiner matrix while still producing the same answer
as the sequential algorithm in expectation.  This projection incurs a
space and time overhead of $O(z\times k)$ where $z$ is the number of
non-zeros in an example, $x_i$.  This overhead is acceptable for
small $k$ and in fact in our experiments in Section~\ref{sec:eval}, $k$
is between $7$ to $15$ across all benchmarks. Most of the overhead for
such a small $k$ is hidden by utilizing SIMD hardware within a
processor (SymSGD with one thread is only half as slow as the
sequential SGD as discussed in Section~\ref{sec:eval}). After learning
a local model and a projected model combiner in each processor,
\ouralgo combines the resulting local models using the combiners, but
additionally employs the optimizations discussed in
Section~\ref{sec:synciter}.

Note that a subset of the data, $D_k$, often contains a subset of total
number of features.  Our implementation takes advantage of this
property and allocates and initializes \texttt{A} for only these
\emph{observed} features.

\subsection{The Variance of Projection}
\label{sec:synciter}
The unbiased estimation above 
is useful only if the variance of the approximation is acceptably
small. The following lemma describes the variance of the random
projection described above.


The trace of a matrix $M$, $tr(M)$ is the sum of the diagonal
elements.  Let $\lambda_i(M)$ by the $i$th eigenvalue of $M$ and
$\sigma_i(M) = \sqrt{\lambda_i(M^T\cdot M)}$ the $i$th singular value
of $M$. 
Let $\sigma_{max}(M)$ be the maximum singular value of $M$.

\begin{lemma}
\label{lem:boundcov}
Let $v = M \cdot A \cdot A^T \cdot \Delta w$. Then the trace of the
covariance matrix $tr(\covm{v})$ is bounded by
\begin{align*}
tr(\covm{v}) &\geq \frac{\norm{\Delta w}^2}{k} \sum_i \sigma_i^2(M)\\
tr(\covm{v}) &\leq  \frac{\norm{\Delta w}^2}{k} (\sum_i
\sigma_i^2(M)+\sigma_{max}^2(M))\nonumber
\end{align*}
\end{lemma}
\begin{proof}
See Section~\ref{sec:boundcovproof}.
\end{proof}

The covariance is small if $k$, the dimension of the
projected space, is large. But increasing $k$ proportionally increases
the overhead of the parallel algorithm. Similarly, covariance is small if the projection happens on
small $\Delta w$. Looking at Equation~\ref{eq:MA}, this means that $w_{i-1}$
should be as close to $w_s$ as possible, implying that processors
should communicate frequently enough such that their models are
roughly in sync. Finally, the singular values of $M$ should be as
small as possible. The next section describes a crucial optimization that achieves this. 

\subsection{Reducing the Variance}
Equation~\ref{eqn:mc} suggests that when $\alpha_i$ is small, the model combiner
$M_D(w)$ is dominated by the $I$ term. From Lemma~\ref{lemma:lowrank}
in Section~\ref{sec:lowrank} shows that the combiner matrix $M_D(w)$ generated
from $n$ examples, $M_D(w) - I$ has at most $n$ non-zero singular
values. Because each
processor operates on a subset of the data it is likely that $n$
examples $\ll$ $f$ features. We use these observations to lower the
variance of dimensionality reduction by projecting the matrix $N_D$
instead of $M_D$. This optimization is crucial for the scalability of
\ouralgo. 

With this optimization the model combiner update becomes
\begin{align} 
w_i \approx &S_{D_i}(w_g) + w_{i-1} - w_g \nonumber\\
            &+ N_{D_i}(w_g)\cdot A \cdot A^T \cdot (w_{i-1} - w_s) \label{eq:NA}
\end{align}
Lemma~\ref{lem:random} guarantees that the approximation above is unbiased. 

An important factor in controlling the singular values of $N_{D_k}(w_{g})$ is
the frequency of model combinations which is a tunable parameter in
\ouralgo. As it is shown in Appendix~\ref{sec:singular}, with more
communication, the smaller the singular values of $N_{D_k}(w_{g})$ and the
less variance (error) in Equation~\ref{eq:NA}.

\subsection{Empirical Evaluating Singular Values of $M_D(w)$}
\label{sec:singular}
Figure~\ref{fig:eigens} empirically demonstrates the benefit of taking
identity off. This figure plots the singular values of $M_D(w)$ for
RCV1 (described in Section~\ref{sec:eval}) after processing
$64,128,256,512$ examples for logistic and linear regression. As it
can be seen, the singular values are close to $1$.  However, the
singular values of $N_{D}(w) = M_D(w)-I$ are roughly the same as those of
$M_D(w)$ minus $1$ and consequently, are small. Finally, the smaller
$\alpha$ (not shown), the closer the singular values of $M_D(w)$ are
to $1$ and the singular values of $N_{D}(w)$ are close to $0$. Also, note
that the singular values of $M_D(w)$ decrease as the numbers of
examples increase and therefore, the singular values of $N_{D}(w)$
increase. As a result, the more frequent the models are combined, the
less variance (and error) is introduced.  

\begin{figure*}[h]
  \centering
  \includegraphics[width=0.9\textwidth]{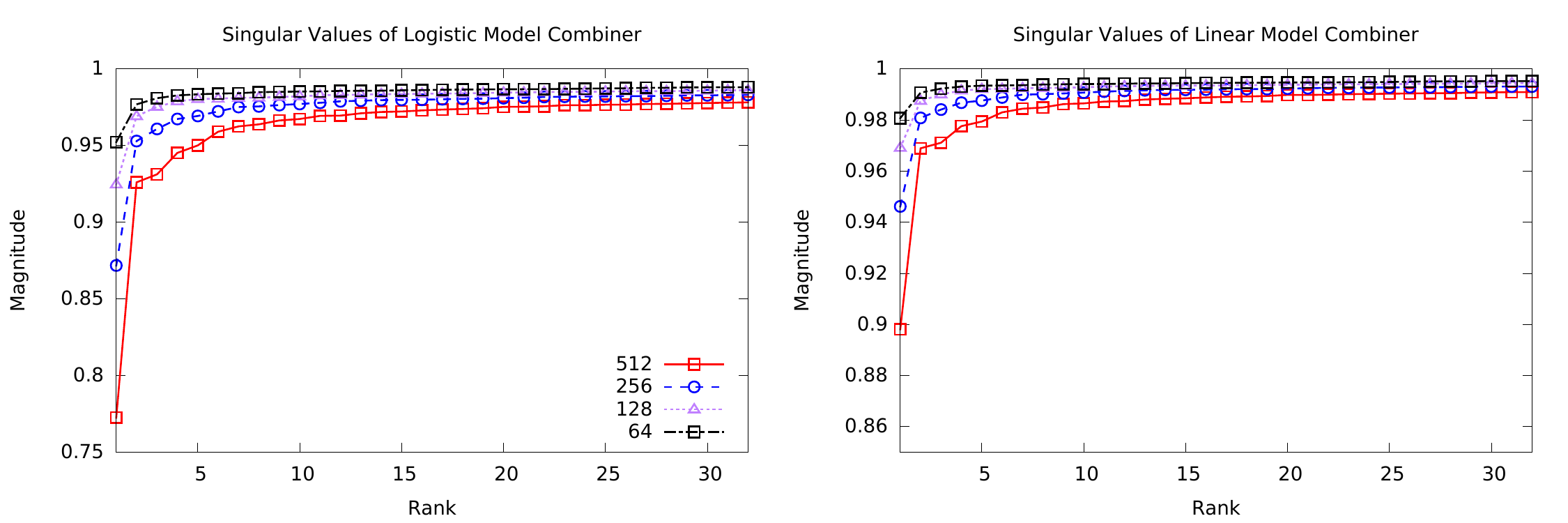}
  \caption{Distribution of singular values of model combiners for RCV1 dataset for logistic regression with 
	$\alpha=0.01$ and for linear regression with $\alpha=0.001$. Different lines correspond to different 
	block sizes.}
  \label{fig:eigs}
\end{figure*}


\section{Parallel \ouralgo Implementation}
\label{sec:impl}
\begin{table*}
\centering
\resizebox{\textwidth}{!}{
\begin{tabular} { c | c | c | c | c | c | c | c | c }
\multirow{2}{*}{Dataset} & \multirow{2}{*}{\#Feat} & \multirow{2}{*}{\#Examples} & \multirow{2}{*}{Average NNZ} & Average NFNZ & \multicolumn{2}{c|}{AUC} & \multicolumn{2}{c}{SymSGD speedup over Hogwild} \\
\cline{6-9}
                         &                         &                             &                              & Ratio        & Logistic & Linear        & Logistic & Linear \\
\hline \hline
  RCV1    & 47153    & 781265  & 74.71   & 0.219 & 0.9586 & 0.959  & 2.60 & 2.60 \\
  AdClick & 3076     & 499980  & 969.38  & 0.947 & 0.7443 & 0.7654 & 2.99 & 2.94 \\
  Epsilon & 2000     & 400000  & 2000    & 1.00  & 0.9586 & 0.959  & 2.55 & 2.45 \\
  URL     & 3231961  & 1677282 & 111.62  & 0.765 & 0.9991 & 0.9986 & 1.90 & 1.04 \\  
  Criteo  & 1703961  & 1000000 & 33.32   & 0.530 & 0.7627 & 0.7633 & 2.05 & 1.91 \\
  Webspam & 16609143 & 279999  & 3727.75 & 0.693 & 0.9992 & 0.9909 & 1.43 & 1.78
\end{tabular}}
\caption{Datasets characteristics.}
\label{tab:datasets}
\end{table*}

This section discusses the \ouralgo
implementations. Section~\ref{sec:parallel:algo} gives a general
specification of a parallel SGD algorithm where
Section~\ref{sec:model:combiners} describes how to build model
combiners.  There are many ways to implement these general
specifications and in this section we discuss some of our
implementation strategies for shared-memory machine.

Section~\ref{sec:parallel:algo} describes a map-reduce style version
of \ouralgo which we call \MR.  In contrast to our algorithm, \HW
asynchronously updates the model parameters.  Because \MR requires
computing model combiners, it does strictly more work than \HW and is
thus a constant factor slower, theoretically. However, even sparse
datasets have a frequently used subset of features which are likely to
show up in many input examples and as we show in
Section~\ref{sec:eval}, this frequent subset causes scalability issues
for \HW. When cache-lines are invalidated across sockets, which
happens often for these frequently accessed subset, \HW incurs large
overheads which limit its scalability.

\AS is a hybrid implementation of \ouralgo which blends asynchronous
updates of infrequent model parameters with \MR style updates for the
frequent ones.  Because the frequently accessed subset of features is
often much smaller than the infrequently accessed ones, \AS has
low-overhead, like \HW. However, because cache-lines are not
invalidated as often, it scales to multiple sockets.

The 5th column in Table~\ref{tab:datasets} (Average NFNZ Ratio) shows
the average number of frequent features in each input example divided
by the number of non-zero features in that input example.  A value of
0 means all features are infrequent and 1 means all features are
frequent. We define a frequent feature as to whether a particular
feature shows up in at least 10\% of the input examples. At runtime,
\AS samples 1000 input examples to find frequent features and builds a
model combiner for that subset and asynchronously updates those
features not in that subset.

\para{Frequency of Model Combination}
\label{sec:syncfreq}
Equation~\ref{tayloreqn} shows that the error in \ouralgo is dependent
on the norm of $\Delta w$; the smaller the norm of $\Delta w$, the
less the error. The way that we control the norm of $\Delta w$ is by
limiting the number of examples that each processor sees before it
combines its local model with the global model. We call this parameter
the {\em block size}.  The trade-offs of high and low values of block
size are clear: large block size allows the \ouralgo communicate less
often and improve overall running time but can potentially suffer in
accuracy due to size of $\Delta w$. On the other hand, low values for
block size enjoys better convergence but the overhead of model
combination may affect the performance.

Block size is set to a constant value per benchmark (part of a
parameter sweep discussed in Section~\ref{sec:impl}) throughout the
execution of \ouralgo.  In future work we expect to dynamically adjust
when to communicate by measuring the norm of $\Delta w$.

\para{Details}
While, in theory, the computational complexity of computing a model
combiner $O(k)$ (where $k < 15$ in all experiments), we do not see a
$k\times$ slowdown. Each processor consecutively stores each of the
$k$ vectors in $A$ so \ouralgo can exploit good cache locality in
addition to SIMD units.  This is apparent in our experiments:
Figure~\ref{fig:log} shows that the difference between \AS and \HW at
1 processor is almost 0 even though the former does $k$ times more
work than the latter.

Lastly, \ouralgo uses a sparse projection~\cite{verysparse} to further
reduce the overhead of computing $A$. Each element of $A$ is
independently chosen from $\{\frac{1}{3}, -\frac{1}{3},0\}$ with
probability $\{\frac{1}{6}, \frac{1}{6}, \frac{2}{3}\}$,
respectively. This approach sparsifies $A$ but it still satisfies
Lemma~\ref{lem:random}.

\section{Evaluation}
\label{sec:eval}
\newcommand{\baseline}{Baseline\xspace}
\newcommand{\hwpaper}{HW-Paper\xspace}
\newcommand{\hwrel}{HW-Release\xspace}
\newcommand{\hwour}{HogWild\xspace}

\begin{figure*}[h]
  \centering
  \includegraphics[width=0.9\textwidth]{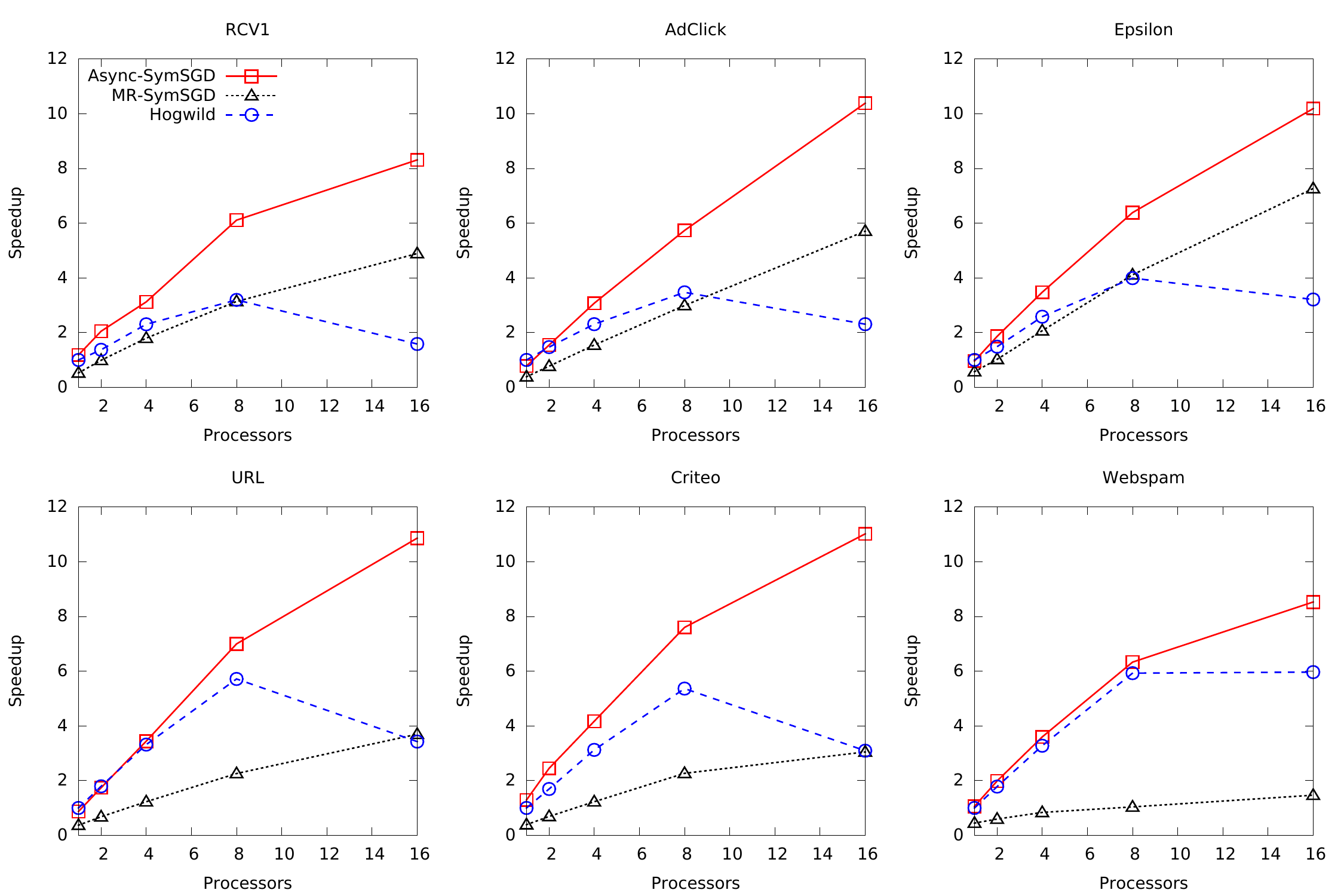}
  \caption{Speedup of logistic regression training when using \AS, \MR, and \HW on a 16 core machine.}
  \label{fig:log}
\end{figure*}

All experiments described in this section were performed on an Intel
Xeon E5-2630 v3 machine clocked at 2.4 GHz with 256 GB of RAM. The machine has two
sockets with 8 cores each, allowing us to study the scalability of the
algorithms across sockets. We disabled hyper-threading and turbo boost.
We also explicitly pinned threads to cores in a compact way which means
that thread $i+1$ was placed as close as possible to thread $i$.
The machine runs Windows 10. All of our implementations were compiled
with Intel C/C++ compiler 16.0 and relied heavily on OpenMP primitives
for parallelization and MKL for efficient linear algebra computations.
And, finally, to measure runtime, we use the average of five
independent runs on an otherwise idle machine.

\para{Algorithms}
Section~\ref{sec:impl} discusses how we implement three SGD
algorithms: \AS, \MR, and \HW. For each, we experimented with ordinary
least squares (OLS) regression and logistic regression (See
Table~\ref{tab:model:combiners} for the model combiners). This section
presents results for logistic regression.  The results for OLS are
similar so we present them in Appendix~\ref{sec:ols}.

When studying the scalability of a parallel algorithm, it is important
to compare the algorithms against an efficient
baseline~\cite{hpc-tenpitfalls,frankmcsherry}. Otherwise, it is
empirically not possible to differentiate between the scalability
achieved from the parallelization of the inefficiencies and the
scalability inherent in the algorithm. We spent a significant effort
to implement a well-tuned version of all algorithms. For example, \AS,
\MR, and \HW with 1 thread are between $1.5\times$ to $4.2\times$ faster
than \vowpal~\cite{vowpal}, a widely used public library.

\para{Datasets}
Table~\ref{tab:datasets} describes various statistics of each
benchmark. They are all freely available, with the exception of
AdClick, which is an internal Ad dataset. For each algorithm and benchmark, we did parameter sweep
over the learning rate, $\alpha$, and picked that $\alpha$ which gave
the best AUC after 10 passes over the data.  For \AS and \MR, we then
fixed $\alpha$ and swept over {\em block size} and $k$ and picked the
configuration which maintained sequential accuracy up to the fourth
digit.

\para{Results}
The last two columns of Table~\ref{tab:datasets} summarize the speedup
of \AS over \HW for both logistic and linear regression.  \AS is, on
average, 2.25X faster than \HW.  The reason is that \AS is able to
scale to multiple sockets: cache-traffic from the frequent subset of
each example causes \HW to suffer scalability when moving from 8 to 16
cores.  Figure~\ref{fig:log} shows this phenomenon in greater
detail. A point on this graph (x-axis, y-axis) shows the speedup of
\AS, \MR, and \HW, respectively (y-axis) as a function of the number
of threads (x-axis).  In all benchmarks, \HW is slower on 16 threads
than 8 (with the exception of Webspam wherein performance stays
roughly constant).  In contrast both \AS and \MR scale across sockets
roughly linearly.  Because \AS uses a model combiner only for those
frequently accessed subset of features, its overhead is lower than \MR
and is thus consistently faster.  The results are similar for linear
regression with the exception of URL: \AS stops scaling at 8 threads,
like \HW.

\section{Related Work}
Most schemes for parallelizing SGD learn
local models independently and communicate to update the global model.
The algorithms differ in how and how often the update is performed.
These choices determine the applicability of the algorithm to
shared-memory or distributed systems.

To the best of our knowledge, our approach is the only one 
that seeks to retain the semantics of the sequential
SGD algorithm. 
Given a tight coupling of the processing units, Langford et
al.~\cite{langford2009slow} suggest on a round-robin scheme to update
the global model allowing for some staleness. However, as the SGD
computation per example is usually much smaller when compared to the
locking overhead, \hogwild~\cite{hogwild} improves on this approach to
perform the update in a ``racy'' manner. While \hogwild is
theoretically proven to achieve good convergence rates provided the
dataset is sparse enough and the processors update the global model
fast enough, our experiments show that the generated cache-coherence
traffic limits its scalability particularly across multiple sockets.
Lastly, unlike \ouralgo, which works for both
sparse and dense datasets, \hogwild is explicitly designed for sparse data.
Recently, \cite{hogbatch} proposed applying lock-free \hogwild approach
to mini-batch. However, mini-batch converges slower than SGD and also 
they did not study multi-socket scaling.

Zinkevich et al.~\cite{naiveparallel} propose a MapReduce-friendly
framework for SGD.  The basic idea is for each machine/thread to run a
sequential SGD on its local data. At the end, the global model is
obtained by averaging these local models. Our experiments with this
approach show it converges very slow in comparison to a sequential
algorithm because the model parameters derived from sparse features
are penalized by that average at every step. Alekh et
al.~\cite{allreduce} extend this approach by using MPI\_AllReduce
operation. Additionally, they use the adagrad~\cite{adagrad} approach
for the learning rates at each node and use weighted averaging to
combine local models with processors that processed a feature more
frequently having a larger weight. Our experiments on our datasets and
implementation shows that it does not achieve the sequential accuracy
for similar reasons as Zinkevich et al.

Several distributed frameworks for machine learning are based on
parameter server~\cite{parameterServerNIPS, parameterServerOSDI} where
clients perform local learning and periodically send the changes to a
central parameter server that applies the changes. For additional
parallelism, the models themselves can be split across multiple
servers and clients only contact a subset of the servers to perform
their updates.

Lastly, there is a significant body of work in the high-performance
computing literature on linear solvers.  For example, MKL has
optimized routines for dense linear least squares
problems~\cite{mkl}.  We found these routines to be significantly
slower than even our sequential baseline running OLS on dense datasets
and MKL does not deal with non-linear terms nor sparse data.
Likewise, randomized numerical linear algebra methods, like RandNLA,
use random projections to solve linear least squares problems
quickly~\cite{Drineas}.  While both our technique and RandNLA use
randomized projections, our insight of taking $I$ off of the matrix we
project is a critical step to controlling the accuracy of our
approach.  Further, RandNLA is specific to linear least squares.


\section{Conclusion}
With terabytes of memory available on multicore machines today, our
current implementation has the capability of learning from large
datasets without incurring the communication overheads of a
distributed system. That said, we believe the ideas in this paper
apply to distributed SGD algorithms we plan to pursue in future work.

\bibliography{paper}
\bibliographystyle{icml2017}

\label{lem:proof}
\label{sec:singular}
\label{fig:eigen1}
\label{fig:eigen2}
\label{fig:eigens}
\label{sec:ols}
\label{fig:lin}
\label{appendix:algos}
\label{appendix:speedups}
\label{fig:aloispeed}
\label{fig:sectorspeedup}
\label{fig:mnistspeed}
\label{fig:mnist8mspeed}
\label{fig:news20speed}
\label{fig:urlspeed}
\label{fig:speedups}
\label{sec:boundcovproof}
\label{sec:convproof}
\label{eqn:oursgd}
\label{eqn:taylor}
\label{ass:convex}
\label{ass:boundgrad}
\label{fr-bound}
\label{sr-bound}
\label{cexpleq}
\label{sec:lowrank}
\label{lemma:lowrank}
\onecolumn
\section{Appendix}
\subsection{Ordinary Least Squares Regression Results}
\label{sec:ols}
\begin{figure*}[h]
  \centering
  \includegraphics[width=0.9\textwidth]{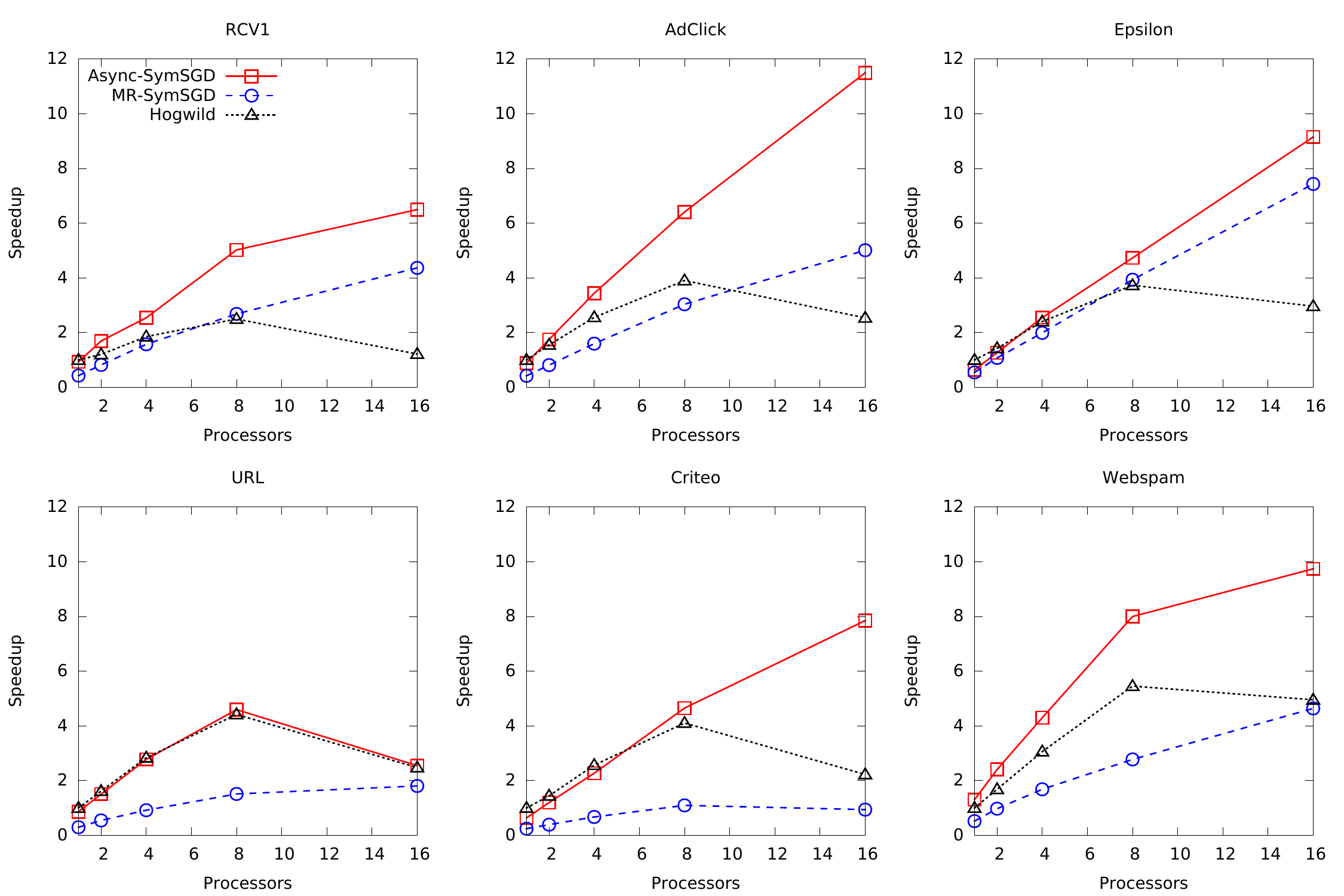}
  \caption{Speedup of Least Squares regression training when using \AS, \MR, and \HW on a 16 core machine.}
  \label{fig:lin}
\end{figure*}

\newcommand{\cov}[1]{\mathrm{Cov}(#1)}
\newcommand{\rank}{{\rm rank}}

\subsection{Variance and Covariance of $\frac{1}{r}M\cdot A\cdot A^T \cdot \Delta w$}
In here, for the sake of simplicity, we use $w$ instead of $\Delta w$
and instead of $k$ for the size of the projected space, we use $r$ since $k$ is used for summation 
indices in here, heavily.
We want to estimate $v=M\cdot w$ with $\frac{1}{r} M\cdot A\cdot A^T\cdot w$, where $A$ is a $f
\times r$ matrix, where $a_{ij}$ is a random variable with the following properties: 
$\Expect(a_{ij}) = 0$, $\Expect(a_{ij}^2) = 1$, and $\Expect(a_{ij}^4) = \rho = 3$. 

Let $m^T_s$ be some row of $M$. Its estimation in $M\cdot w$ is $v_s =
\frac{1}{r}\cdot m^T_s\cdot A\cdot A^T\cdot w$. From Lemma~\ref{lem:random} $\Expect(v_s) = m^T_s\cdot w$. 


We will use the notation $ij = kl$ to mean $i=k \wedge j=l$, and $ij
\neq kl$ to mean its negation. 
Let $m_s$, $m_{t}$ be two rows of $M$. We want to find the
covariance of the resulting $v_s$ and $v_{t}$. 

\begin{align*}
&r^2 \cdot \Expect(v_s,v_{t}) &\\
&= r^2\cdot \Expect(\frac{1}{r^2} \sum_{i,j,k} m_{si} a_{ij} a_{kj} w_k \cdot\sum_{i',j',k'} m_{ti'} a_{i'j'} a_{k'j'} w_{k'})\\
&= \sum_{i,j,k,i',j',k'} m_{si} m_{ti'} w_k w_{k'} \Expect(a_{ij} a_{kj} a_{i'j'} a_{k'j'})\\
&=\sum_{i,j,k,i',j',k':ij=kj=i'j'=k'j'} m_{si} m_{ti'} w_k w_{k'} \Expect(a_{ij} a_{kj} a_{i'j'} a_{k'j'})\\
&+\sum_{i,j,k,i',j',k':ij=kj\neq i'j'=k'j'} m_{si} m_{ti'} w_k w_{k'} \Expect(a_{ij} a_{kj} a_{i'j'} a_{k'j'}) \\
&+\sum_{i,j,k,i',j',k':ij=i'j'\neq kj=k'j'} m_{si} m_{ti'} w_k w_{k'} \Expect(a_{ij} a_{kj} a_{i'j'} a_{k'j'}) \\
&+\sum_{i,j,k,i',j',k':ij=k'j'\neq i'j'=kj} m_{si} m_{ti'} w_k w_{k'} \Expect(a_{ij} a_{kj} a_{i'j'} a_{k'j'}) 
&\text{as terms with $\Expect(a_{ij})$ cancel out}\\
&= \sum_{i,j} m_{si} m_{ti} w_i w_i \rho 
+ \sum_{i,j,i',j':ij\neq i'j'} m_{si} m_{ti'} w_i w_{i'} \\
&+ \sum_{i,j,k:i\neq k} m_{si} m_{ti} w_k w_k
+\sum_{i,j,k:i\neq k} m_{si} m_{tk} w_k w_i
&\text{as $\Expect(a_{ij}a_{kl}) = 1$ when $ij \neq kl$}\\
&= \rho \sum_{i,j} m_{si} m_{ti} w_i^2 \\
&+ \sum_{i,j,i',j'} m_{si} m_{ti'} w_i w_{i'} - \sum_{i,j,i',j':ij
  = i'j'} m_{si} m_{ti'} w_i w_{i'} \\
&+ \sum_{i,j,k} m_{si} m_{ti} w_k^2 - \sum_{i,j,k:i=k} m_{si}
m_{ti} w_k^2 \\
&+\sum_{i,j,k} m_{si} m_{tk} w_k w_i  - \sum_{i,j,k:i=k} m_{si}
m_{tk} w_k w_i\\
&= (\rho-3) \sum_{i,j} m_{si} m_{ti} w_i^2 
+ \sum_{i,j,i',j'} m_{si} m_{ti'} w_i w_{i'} \\
&+ \sum_{i,j,k} m_{si} m_{ti} w_k^2 
+\sum_{i,j,k} m_{si} m_{tk} w_k w_i \\
&= r^2\sum_{i,i'} m_{si} m_{ti'} w_i w_{i'} 
+ r\sum_{i,k} m_{si} m_{ti} w_k^2 
+ r\sum_{i,k} m_{si} m_{tk} w_i w_k & \text{as $\rho=3$ and
  $j\in[1\dots k]$}\\
&= (r^2+r)\sum_{i,i'} m_{si} m_{ti'} w_i w_{i'} + r\cdot
m_s^T\cdot m_{t}\sum_k w_k^2\\
\end{align*}

In other words
$$ \Expect(v_sv_{t}) = (1+\frac{1}{r})\sum_{i,i'} m_{si} m_{ti'} w_i w_{i'} +
\frac{1}{r}\cdot m_s^T\cdot m_{t}\sum_k w_k^2$$
The covariance $\cov{a,b} = \Expect(a\cdot b) - \Expect(a)\Expect(b)$. Using this we have

\begin{align*}
&\cov{v_s, v_{t}} \\
&=(1+\frac{1}{r})\sum_{i,i'} m_{si} m_{ti'} w_i w_{i'} +
\frac{1}{r}\cdot m_s^T\cdot m_{t}\sum_k w_k^2
-\Expect(v_s)\Expect(v_{t})\\
&= (1+\frac{1}{r})\sum_{i,i'} m_{si} m_{ti'} w_iw_{i'} 
+ \frac{1}{r}\cdot m_s^T\cdot m_{t}\sum_k w_k^2
-\Expect(v_s)\Expect(v_{t})\\
&= 
 (1+\frac{1}{r}) \Expect(v_s)\Expect(v_{t})
+ \frac{1}{r}\cdot m_s^T\cdot m_{t}\sum_k w_k^2
-\Expect(v_s)\Expect(v_{t})\\
&= 
\frac{1}{r}\Expect(v_s)\Expect(v_{t})
+ \frac{1}{r}\cdot m_s^T\cdot m_{t}\sum_k w_k^2\\
&= 
\frac{1}{r}\Expect(v_s)\Expect(v_{t})
+ \frac{1}{r}\cdot (M\cdot M^T)_{st}\norm{w}^2\\
&= 
\frac{1}{r}(M\cdot w)_s (M\cdot w)_t
+ \frac{1}{r}\cdot (M\cdot M^T)_{st}\norm{w}^2\\
&= 
\frac{1}{r}((M\cdot w)\cdot (M\cdot w)^T)_{st}
+ \frac{1}{r}\cdot (M\cdot M^T)_{st}\norm{w}^2\\
\end{align*}

Let $\covm{v}$ be the covariance matrix of $v$.  
That is, $\covm{v}_{ij} =
\cov{v_i, v_j}$.
So, we have
$$ \covm{v} = \frac{1}{r}(M\cdot w)\cdot (M\cdot w)^T
+ \frac{1}{r}(M\cdot M^T)\norm{w}^2
$$

Note that we can use this computation for matrix $N=M-I$ as well since we did not assume anything about the matrix $M$
from the beginning. 
Therefore, for $v'=w + \frac{1}{r} N\cdot A\cdot A^T\cdot w$, $\covm{v'}=
\frac{1}{r}(N\cdot w)\cdot (N\cdot w)^T + \frac{1}{r}(N\cdot N^T)\norm{w}^2$ since $w$ is a constant in $v'$ and
$\covm{a+x}=\covm{x}$ for any constant vector $a$ and any probabilistic vector $x$. Next we try to bound
$\covm{v}$.

\subsection{Proof of Lemma~\ref{lem:boundcov}}
\label{sec:boundcovproof}
We can bound $\covm{v}$ by computing its trace since $tr(\covm{v})=\sum_i var(v_i)$, the summation 
of the variance of elements of $v$.
\begin{align*}
tr(\covm{v})
&= \frac{1}{r}tr((M\cdot w)\cdot (M\cdot w)^T) + \frac{1}{r}\norm{w}^2 tr(MM^T)\\
&=\frac{1}{r}\norm{M\cdot w}^2 + \frac{1}{r}\norm{w}^2 \Big(\sum_i \lambda_i(M\cdot M^T)\Big)\\
&=\frac{1}{r}\norm{M\cdot w}^2 + \frac{1}{r}\norm{w}^2 \Big(\sum_i \sigma_i(M)^2\Big)
\end{align*}
where $\lambda_i{M\cdot M^T}$ is the $i^{th}$ largest eigenvalue of $M\cdot M^T$ which is the square of 
$i^{th}$ largest singular value of $M$, $\sigma_i(M)^2$. Since 
$\norm{M\cdot w}^2\leq \norm{w}^2\norm{M}^2=\norm{w}^2 \sigma_{max}(M)^2$, we can bound $tr(\covm{v})$ as follows:
$$tr(\covm{v})\leq \frac{1}{r}(\sigma_{max}(M)^2) + \frac{1}{r}\norm{w}^2 \Big(\sum_i \sigma_i(M)^2\Big)$$

It is trivial to see that:
$$
\frac{1}{r}\norm{w}^2 \Big(\sum_i \sigma_i(M)^2\Big) \leq tr(\covm{v})
$$

Combining the two inequalities, we have:
$$
\frac{1}{r}\norm{w}^2 \Big(\sum_i \sigma_i(M)^2\Big) \leq tr(\covm{v}) \frac{1}{r}(\sigma_{max}(M)^2) + \frac{1}{r}\norm{w}^2 \Big(\sum_i \sigma_i(M)^2\Big)
$$

The same bounds can be derived when $N=M-I$ is used.

\subsection{Rank of Matrix $N = M-I$}
\label{sec:lowrank}
Now we show that subtracting $I$ from a model combiner results in a
matrix with small rank. Thus, most of its singular values are zero. We
assume that the model combiner is generated for a linear learner and
thus it is of the form $\prod_i(I - \alpha x_i x_i^T)$ where any
nonlinear scalar terms from the Hessian are factored into $\alpha$.

\begin{lemma}
For the matrix $M_{a\rightarrow b} =\prod_{i=b}^a (I-\alpha x_i \cdot x_i^T)$, $\rank(M_{a\rightarrow b}-I)\leq b-a$.
\label{lemma:lowrank}
\end{lemma}
\begin{proof}
The proof is by induction. The base case is when $a=b$ and $M_{a\rightarrow b}=I$. It is clear that $I-I=0$ which is
of rank zero. For the inductive step, assume that $\rank(M_{a\rightarrow b-1}-I) \leq b-a-1$. We have

\begin{align*}
M_{a\rightarrow b}-I&=(I-\alpha x_b\cdot x_b^T)M_{a\rightarrow b-1}-I\\
&=(M_{a\rightarrow b-1}-I)-\alpha x_b\cdot (x_b^T\cdot M_{a\rightarrow b-1})
\end{align*}

Term $\alpha x_b\cdot (x_b^T\cdot M_{a\rightarrow b-1})$ is a rank-1 matrix and term $(M_{a\rightarrow b-1}-I)$
is of rank $b-a-1$ by induction hypothesis. Since for any two matrices $A$ and $B$, $\rank(A+B)\leq \rank(A)+\rank(B)$,
$\rank(M_{a\rightarrow b}-I)\leq \rank(M_{a\rightarrow b-1}) + \rank(-\alpha x_b\cdot (x_b^T\cdot M_{a\rightarrow b-1}))
 \leq b-a-1+1=b-a$.
\end{proof}

\subsection{Convergence Proof}
\label{sec:convproof}
\newtheorem{assumption}{Assumption}
Let the sequence $w_0, w_1, \ldots w_t$ represent the sequence of
weight vectors produced by a sequential SGD run. We know that this
sequence converges to the desired minimum $w^*$. Our goal is to show
that \ouralgo also converges to $w^*$. Consider a process
processing example sequences $D$ starting with model $w_t-\Delta w$
that is $\Delta w$ different from the ``true'' model $w_t$ that a
sequential SGD would have started with. The output of this processor
is 
\begin{equation}
\label{eqn:oursgd}
w_{t+1} = S_D(w_t-\Delta w) + M_D\Delta w
\end{equation}
where the model combiner after the projection by taking the I off is
given by 
$$M_D = I + (S'_D(w_t - \Delta w) - I)AA^T$$
Applying Taylor's theorem, we have for some $0 \leq \mu \leq 1$
\begin{equation}
\label{eqn:taylor}
S_D(w_t) = S_D(w_t-\Delta w) + S'_D(w_t-\Delta w)\Delta w
+ \frac{1}{2}\Delta w^TS_D''(w_t-\mu\Delta w) \Delta w
\end{equation}

Comparing Equation~\ref{eqn:taylor} with Equation~\ref{eqn:oursgd}, we see that
\ouralgo introduces two error terms to a sequential SGD
$$ w_{t+1} = S_D(w_t) + FR_D(w_t, \Delta w) + SR_D(w_t,
\Delta w)$$
where the first-order error term $FR$ comes due to the projection approximation
$$FR_D(w_t, \Delta w) = (I- S'_D(w_t-\Delta w))(I - AA^T)$$
and the second-order error term $SR$ comes due to neglecting the
higher-order terms in the Taylor expansion. 
$$ SR_D(w_t, \Delta w)= \frac{1}{2}\Delta w^T S_D''(w_t-\mu\Delta w) \Delta w$$

To prove convergence \ouralgo, we show that SGD is ``robust'' with
respect to adding these error terms. The proof follows along the same
lines as the convergence proof of SGD by Bottou~\cite{bottou} and uses
similar notations. We state below the assumptions and Lemmas required
for the main proof. The proof of these Lemmas is shown later. 

\begin{assumption}
\label{ass:convex}
Convexity of the cost function
$$(w - w^*).G(w) > 0$$ for $w \neq w*$.
\end{assumption}

\begin{assumption}
\label{ass:boundgrad}
Bounded gradients. For any input $z = (X,y)$
$$ \norm{G_z(w)} \leq b_G\norm{w - w^*}$$ 
for some $b_G \geq 0$.
\end{assumption}

\begin{lemma}
\label{fr-bound}
Bounds on the mean and second moment of $FR$
\begin{align*}
E_A(FR_D(w_t, \Delta w)) & = 0\\
E_A(\norm{FR_D(w_t, \Delta w)}^2) & \leq b_{FR} \norm{w_t - w^*}^2
\end{align*}
for some $b_{FR} \geq 0$
\end{lemma}

\begin{lemma}
\label{sr-bound}
Bounds on $SR$
\begin{align*}
\norm{SR_D(w_t, \Delta w)} \leq b_{SR} \norm{w_t - w^*}
\end{align*}
for some $b_{FR} \geq 0$
\end{lemma}

Convergence of \ouralgo follows if the following sequence converges
almost surely to $0$. 
$$h_{t} = \norm{w_t - w^*}^2$$

We assume the worst case where the error
terms are added every step of the SGD. This make the proof much
simpler and more along the lines of the proof in
Bottou~\cite{bottou}. Note that this is indeed the worst case as the
error bounds from Lemma~\ref{fr-bound} and Lemma~\ref{sr-bound} are
for arbitrary steps. 

\begin{theorem}
The sequence $h_t$ converges to $0$ almost surely.
\end{theorem}
\begin{proof}
As in Bottou~\cite{bottou}, we denote $\mathcal{P}_t$ denote all the
random choices made by the algorithm at time $t$. For terseness, we
use the following notation for the conditional expectation with
respect to $\mathcal{P}_t$:
$$CE(x) = E(x | \mathcal{P}_t)$$

The key technical challenge is in showing that the infinite sum of the
positive expected variations in $h_t$ is bounded, which we show
below. Let $z = (X,y)$ be the example processed at time $t$. We use
following short hand.
$$R_z(w_t, \Delta w) = FR_z(w_t, \Delta w) + SR_z(w_t, \Delta w)$$. 
\begin{align*}
&CE(h_{t+1} - h_t)\\ 
&= -2\gamma_t(w_t - w^*) CE(G_z(w_t) + R_z(w_t, \Delta w)) + \gamma_t^2CE(\norm{G_z(w_t) + R_z(w_t, \Delta w)}^2)\\
&= -2\gamma_t(w_t - w^*) (G(w_t) + CE(R_z(w_t, \Delta w))) + \gamma_t^2CE(\norm{G_z(w_t) + R_z(w_t, \Delta w)}^2)\\
&\leq -2\gamma_t(w_t - w^*) CE(R_z(w_t, \Delta w)) +
\gamma_t^2CE(\norm{G_z(w_t) + R_z(w_t, \Delta w)}^2)\\
&\tag{\text{from Assumption~\ref{ass:convex}}}\\
&\leq -2\gamma_t b_{SR}\norm{w_t - w^*}^2 +
\gamma_t^2CE(\norm{G_z(w_t) + R_z(w_t,\Delta w)}^2) \\
&\tag{\text{from Lemmas~\ref{fr-bound} and \ref{sr-bound}}}\\
&\leq \gamma_t^2CE(\norm{G_z(w_t) + R_z(w_t,\Delta w)}^2)\\
& = \gamma_t^2 (CE(\norm{G_z(w_t)}^2) + CE(\norm{R_z(w_t,\Delta w)}^2) + 2CE(G_z(w_t)R_z(w_t,\Delta w)))\\
& \leq \gamma_t^2 ((b_G + b_{FR} + b_{SR})\norm{w_t - w^*}^2 \\
& \hspace{3cm}+ 2 CE(FR_z(w_t, \Delta w) SR_z(w_t,\Delta w)) + 2CE(G_z(w_t)R_z(w_t,\Delta w)))\\
&\tag{\text{from Assumption~\ref{ass:boundgrad} and Lemmas~\ref{fr-bound} and \ref{sr-bound}}}\\
& \leq \gamma_t^2 ((b_G + b_{FR} + b_{SR})\norm{w_t - w^*}^2 + 2CE(G_z(w_t)SR_z(w_t,\Delta w)))\\
&\tag{\text{as $FR$ has a zero mean~(Lemma~\ref{fr-bound}) and $G, SR$
do not depend on $A$}}\\
& \leq \gamma_t^2 ((b_G + b_{FR} + b_{SR} + 2b_G b_{SR})\norm{w_t - w^*}^2)\\
&\tag{\text{from Assumption~\ref{ass:boundgrad} and Lemma~\ref{sr-bound}}}\\
\end{align*}
In other words, for $B = b_G + b_{FR} + b_{SR} + 2b_G b_{SR}$, we have
\begin{align}
\label{cexpleq}
CE(h_{t+1} - (1 + \gamma_t^2B)h_t) \leq 0
\end{align}
From here on, the proof proceeds exactly as in Bottou~\cite{bottou}. Define auxiliary sequences 
$\mu_t = \Pi_{i=1}^t \frac{1}{1+ \gamma_i^2 B}$ and $h'_t = \mu_t
h_t$. Assuming $\Sigma_t \gamma_t^2 < \infty$, $\mu_t$ converges to a
nonzero value. Since Equation~\ref{cexpleq} implies
$CE(h'_{t+1}-h'_{t}) \leq 0$, from quasi-martingale
convergence theorem, $h_t'$ and thus $h_t$ converges almost
surely. Under the additional assumption that $\Sigma_t \gamma_t =
\infty$, we can show that this convergence is to $0$. 
\end{proof}

The proof above crucially relies on lemmas~\ref{fr-bound} and
\ref{sr-bound} that we now prove. But first we make some assumptions
and prove supplementary lemmas. We restrict the discussion, as in
Lemma~\ref{lemma:lowrank}, to linear learners and that the model
combiners are of the form $M_z(w) = (I - \alpha H_z(w) x x^T)$ for a
scalar Hessian $H_z(w)$

\begin{assumption}
\label{ass:bdw}
\ouralgo synchronizes sufficiently enough so that $\Delta w$ does not grow too large.
$$\norm{\Delta w} \leq min(1, b_{\Delta w}\norm{w_t - w^*})$$
for some $b_{\Delta w} > 0$
\end{assumption}

\begin{assumption}
\label{ass:bh}
Bounded Hessian. 
$$|H_z(w)| \leq b_H$$ for some $b_H > 0$
\end{assumption}

\begin{lemma}
\label{lem:eigen}
The model combiner $M_D(w) = \Pi_i (I - \alpha H_z(w) x_i x_i^T)$ has
bounded eigenvalues
\end{lemma}
\begin{proof}
The proof follows from induction on $i$ using Assumption~\ref{ass:bh}.
\end{proof}

\subsection{Proof of Lemma~\ref{fr-bound}}
The mean is a simple restatement of Lemma~\ref{lem:random}. The second
moment follows from Assumption~\ref{ass:bdw}, Lemma~\ref{lem:eigen}
applied to $M$ and $M^T$, and from Lemma~\ref{lem:boundcov}.

\subsection{Proof of Lemma~\ref{sr-bound}}
For linear learners, we have 
\begin{align*}
S_z(w) &= w - \gamma \dotp{G_z(\dotp{x}{w}, y)}{x}\\
\frac{\partial S_z(w)}{\partial w} &= I - \gamma H_z(\dotp{x}{w}, y)x x^T\\
\frac{\partial^2 S_z(w)}{\partial w^2} &= H'_z(\dotp{x}{w}, y) x \otimes x \otimes x
\end{align*}
where $H$ is the second derivative of the cost with respect to
$\dotp{x}{w}$, and $\otimes$ is the tensor outer product.

In the last equation above, if the input is composed with a previous
SGD phase we have 
\begin{align*}
\frac{\partial^2 S_z(S_D(w))}{\partial w^2} &= H'_z(\dotp{S_D(x)}{w},
y) x \otimes x \otimes (\frac{\partial S_D(w)}{\partial w})^T x
\end{align*}

For notational convenience, let $M_{b\rightarrow a} \triangleq
\prod_{i=b}^a (I - \alpha x_i x_i^T)$. Explicitly differentiating
$S_n'(w)$, we can show that
$$\deriv{S_n'(w)}{w} = (\sum_j (-\alpha_j H_j'(s_j(w))) M_{n
\rightarrow j+1} x_j x_j^T M_{j \rightarrow 1}) \otimes (S_j'(w)^T x)$$
Each element of $SR_z$ is obtained by $\Delta w^T P \Delta w$ where P
is an outer product of a row from the first term above and $S_j'(w)^T
x$. Using Lemma~\ref{lem:eigen} twice we can show that each of these
vectors are bounded. This proves the lemma.

\end{document}